\documentclass[journal]{IEEEtran}
\ifCLASSINFOpdf
   \usepackage[pdftex]{graphicx}
\else
\fi
\usepackage{mathdots}

\usepackage[cmex10]{amsmath}
\usepackage{amssymb}

\newcounter{thm}
\newtheorem{lemma}[thm]{Lemma}
\newtheorem{theorem}[thm]{Theorem}
\newtheorem{corollary}[thm]{Corollary}

\newcommand{\KL}{\mathrm{KL}}
\newcommand{\kl}{\mathrm{kl}}

\begin{document}
%
\title{PAC-Bayesian Inequalities for Martingales}
%
%
%


\author{Yevgeny Seldin, Fran\c{c}ois Laviolette, Nicol{\`o} Cesa-Bianchi, John Shawe-Taylor, Peter Auer%
\thanks{Yevgeny Seldin is with Max Planck Institute for Intelligent Systems, T\"{u}bingen, Germany, and University College London, London, UK. E-mail: seldin@tuebingen.mpg.de}%
\thanks{Fran\c{c}ois Laviolette is with Universit\'{e} Laval, Qu\'{e}bec, Canada. E-mail: francois.laviolette@ift.ulaval.ca}%
\thanks{Nicol{\`o} Cesa-Bianchi is with Dipartimento di Informatica, Universit{\`a} degli Studi di Milano, Milan, Italy. E-mail: nicolo.cesa-bianchi@unimi.it}%
\thanks{John Shawe-Taylor is with University College London, London, UK. E-mail: jst@cs.ucl.ac.uk}%
\thanks{Peter Auer is with Chair for Information Technology, Montanuniversit{\" a}t Leoben, Leoben, Austria. E-mail: auer@unileoben.ac.at}%
}

%
%

\markboth{IEEE TRANSACTIONS ON INFORMATION THEORY,~Vol.~XX, No.~Y, Month~201X}%
{Seldin \MakeLowercase{\textit{et al.}}: PAC-Bayesian Inequalities for Martingales}
%



\maketitle

\begin{abstract}
We present a set of high-probability inequalities that control the concentration of weighted averages of multiple (possibly uncountably many) simultaneously evolving and interdependent martingales. Our results extend the PAC-Bayesian analysis in learning theory from the i.i.d. setting to martingales opening the way for its application to importance weighted sampling, reinforcement learning, and other interactive learning domains, as well as many other domains in probability theory and statistics, where martingales are encountered.

We also present a comparison inequality that bounds the expectation of a convex function of a martingale difference sequence shifted to the $[0,1]$ interval by the expectation of the same function of independent Bernoulli variables. This inequality is applied to derive a tighter analog of Hoeffding-Azuma's inequality.
\end{abstract}

\begin{IEEEkeywords}
Martingales, Hoeffding-Azuma's inequality, Bernstein's inequality, PAC-Bayesian bounds.
\end{IEEEkeywords}

%
\IEEEpeerreviewmaketitle

\section{Introduction}
%
%
%
%
\IEEEPARstart{M}{artingales} are one of the fundamental tools in probability theory and statistics for modeling and studying sequences of random variables. Some of the most well-known and widely used concentration inequalities for individual martingales are Hoeffding-Azuma's and Bernstein's inequalities \cite{Hoe63,Azu67, Ber46}. We present a comparison inequality that bounds the expectation of a convex function of a martingale difference sequence shifted to the $[0,1]$ interval by the expectation of the same function of independent Bernoulli variables. We apply this inequality in order to derive a tighter analog of Hoeffding-Azuma's inequality for martingales.

\begin{figure}
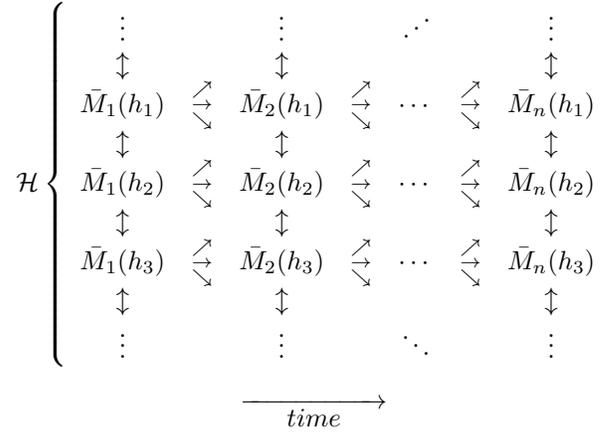

\[
\begin{array}{c}
{\cal H}\left \{ \begin{array}{ccccccc}

\vdots & & \vdots & & \iddots & & \vdots\\

\updownarrow & & \updownarrow & & & & \updownarrow\\

\bar M_1(h_1) & \substack{\nearrow\\\rightarrow\\\searrow} & \bar M_2(h_1) & \substack{\nearrow\\\rightarrow\\\searrow} & \cdots & \substack{\nearrow\\\rightarrow\\\searrow} & \bar M_n(h_1)\\

\updownarrow & & \updownarrow & & & & \updownarrow\\

\bar M_1(h_2) & \substack{\nearrow\\\rightarrow\\\searrow} & \bar M_2(h_2) & \substack{\nearrow\\\rightarrow\\\searrow} & \cdots & \substack{\nearrow\\\rightarrow\\\searrow} & \bar M_n(h_2)\\

\updownarrow & & \updownarrow & & & & \updownarrow\\

\bar M_1(h_3) & \substack{\nearrow\\\rightarrow\\\searrow} & \bar M_2(h_3) & \substack{\nearrow\\\rightarrow\\\searrow} & \cdots & \substack{\nearrow\\\rightarrow\\\searrow} & \bar M_n(h_3)\\

\updownarrow & & \updownarrow & & & & \updownarrow\\

\vdots & & \vdots & & \ddots & & \vdots\\

\end{array} \right .\\
\\
\overrightarrow{~~~~~time~~~~~}
\end{array}
\]
\caption{Illustration of an infinite set of simultaneously evolving and interdependent martingales. ${\cal H}$ is a space that indexes the individual martingales. For a fixed point $h \in {\cal H}$, the sequence $\bar M_1(h), \bar M_2(h), \dots, \bar M_n(h)$ is a single martingale. The arrows represent the dependencies between the values of the martingales: the value of a martingale $h$ at time $i$, denoted by $\bar M_i(h)$, depends on $\bar M_j(h')$ for all $j \leq i$ and $h' \in {\cal H}$ (everything that is ``before'' and ``concurrent'' with $\bar M_i(h)$ in time; some of the arrows are omitted for clarity). A mean value of the martingales with respect to a probability distribution $\rho$ over ${\cal H}$ is given by $\langle \bar M_n, \rho \rangle$. Our high-probability inequalities bound $|\langle \bar M_n, \rho \rangle|$ simultaneously for a large class of $\rho$.}
\label{fig:1}
\end{figure}

More importantly, we present a set of inequalities that make it possible to control weighted averages of multiple simultaneously evolving and interdependent martingales (see Fig. \ref{fig:1} for an illustration). The inequalities are especially interesting when the number of martingales is uncountably infinite and the standard union bound over the individual martingales cannot be applied. The inequalities hold with high probability simultaneously for a large class of averaging laws $\rho$. In particular, $\rho$ can depend on the sample.

One possible application of our inequalities is an analysis of importance-weighted sampling. Importance-weighted sampling is a general and widely used technique for estimating properties of a distribution by drawing samples from a different distribution. Via proper reweighting of the samples, the expectation of the desired statistics based on the reweighted samples from the controlled distribution can be made identical to the expectation of the same statistics based on unweighted samples from the desired distribution. Thus, the difference between the observed statistics and its expected value forms a martingale difference sequence. Our inequalities can be applied in order to control the deviation of the observed statistics from its expected value. Furthermore, since the averaging law $\rho$ can depend on the sample, the controlled distribution can be adapted based on its outcomes from the preceding rounds, for example, for denser sampling in the data-dependent regions of interest. See \cite{SAL+11} for an example of an application of this technique in reinforcement learning.

Our concentration inequalities for weighted averages of martingales are based on a combination of Donsker-Varadhan's variational formula for relative entropy \cite{DV75, DE97, Gra11} with bounds on certain moment generating functions of martingales, including Hoeffding-Azuma's and Bernstein's inequalities, as well as the new inequality derived in this paper.

In a nutshell, the Donsker-Varadhan's variational formula implies that for a probability space $({\cal H}, {\cal B})$, a bounded real-valued random variable $\Phi$ and any two probability distributions $\pi$ and $\rho$ over ${\cal H}$ (or, if ${\cal H}$ is uncountably infinite, two probability density functions), the expected value $\mathbb E_{\rho} [\Phi]$ is bounded as:
\begin{equation}
\mathbb E_{\rho}[\Phi] \leq \KL(\rho\|\pi) + \ln \mathbb E_{\pi} [e^{\Phi}],
\label{eq:basic}
\end{equation}
where $\KL(\rho\|\pi)$ is the KL-divergence (relative entropy) between two distributions \cite{CT91}. We can also think of $\Phi$ as $\Phi = \phi(h)$, where $\phi(h)$ is a measurable function $\phi:{\cal H} \rightarrow \mathbb R$. Inequality \eqref{eq:basic} can then be written using the dot-product notation
\begin{equation}
\langle \phi, \rho \rangle \leq KL(\rho\|\pi) + \ln \left(\langle e^\phi, \pi \rangle \right )
\label{eq:basic-dot}
\end{equation}
and $\mathbb E_\rho[\phi] = \langle \phi, \rho \rangle$ can be thought of as a weighted average of $\phi$ with respect to $\rho$ (for countable ${\cal H}$ it is defined as $\langle \phi, \rho \rangle = \sum_{h \in {\cal H}} \phi(h) \rho(h)$ and for uncountable ${\cal H}$ it is defined as $\langle \phi, \rho \rangle = \int_{\cal H} \phi(h) \rho(h) dh$).\footnote{The complete statement of Donsker-Varadhan's variational formula for relative entropy states that under appropriate conditions $\KL(\rho\|\pi) = \sup_\phi \left (\langle \phi, \rho \rangle - \ln \langle e^{\phi}, \pi \rangle \right)$, where the supremum is achieved by $\phi(h) = \ln \frac{\rho(h)}{\pi(h)}$. However, in our case the choice of $\phi$ is directly related to the values of the martingales of interest and the free parameters in the inequality are the choices of $\rho$ and $\pi$. Therefore, we are looking at the inequality in the form of equation \eqref{eq:basic} and a more appropriate name for it is ``change of measure inequality''.}

The weighted averages $\langle \phi, \rho \rangle$ on the left hand side of \eqref{eq:basic-dot} are the quantities of interest and the inequality allows us to relate all possible averaging laws $\rho$ to a single ``reference'' distribution $\pi$. (Sometimes, $\pi$ is also called a ``prior'' distribution, since it has to be selected before observing the sample.) We emphasize that inequality \eqref{eq:basic-dot} is a deterministic relation. Thus, by a single application of Markov's inequality to $\langle e^\phi, \pi \rangle$ we obtain a statement that holds with high probability for all $\rho$ simultaneously. The quantity $\ln \langle e^\phi, \pi \rangle$, known as the cumulant-generating function of $\phi$, is closely related to the moment-generating function of $\phi$. The bound on $\ln \langle e^\phi, \pi \rangle$, after some manipulations, is achieved via the bounds on moment-generating functions, which are identical to those used in the proofs of Hoeffding-Azuma's, Bernstein's, or our new inequality, depending on the choice of $\phi$.

Donsker-Varadhan's variational formula for relative entropy laid the basis for PAC-Bayesian analysis in statistical learning theory \cite{STW97,ST+98,McA98,See02}, where PAC is an abbreviation for the Probably Approximately Correct learning model introduced by Valiant \cite{Val84}. PAC-Bayesian analysis provides high probability bounds on the deviation of weighted averages of empirical means of sets of independent random variables from their expectations. In the learning theory setting, the space ${\cal H}$ usually corresponds to a hypothesis space; the function $\phi(h)$ is related to the difference between the expected and empirical error of a hypothesis $h$; the distribution $\pi$ is a prior distribution over the hypothesis space; and the distribution $\rho$ defines a randomized classifier. The randomized classifier draws a hypothesis $h$ from ${\cal H}$ according to $\rho$ at each round of the game and applies it to make the prediction on the next sample. PAC-Bayesian analysis supplied generalization guarantees for many influential machine learning algorithms, including support vector machines \cite{LST02, McA03}, linear classifiers \cite{GLLM09}, and clustering-based models \cite{ST10}, to name just a few of them. 

We show that PAC-Bayesian analysis can be extended to martingales. A combination of PAC-Bayesian analysis with Hoeffding-Azuma's inequality was applied by Lever et. al \cite{LLST10} in the analysis of U-statistics. The results presented here are both tighter and more general, and make it possible to apply PAC-Bayesian analysis in new domains, such as, for example, reinforcement learning \cite{SAL+11}.

\section{Main Results}

We first present our new inequalities for individual martingales, and then present the inequalities for weighted averages of martingales. All the proofs are provided in the appendix.

\subsection{Inequalities for Individual Martingales}

Our first lemma is a comparison inequality that bounds expectations of convex functions of martingale difference sequences shifted to the $[0,1]$ interval by expectations of the same functions of independent Bernoulli random variables. The lemma generalizes a previous result by Maurer for independent random variables \cite{Mau04}. The lemma uses the following notation: for a sequence of random variables $X_1,\dots,X_n$ we use $X_1^i := X_1,\dots,X_i$ to denote the first $i$ elements of the sequence.

\begin{lemma}
\label{lem:Martin}
Let $X_1,\dots,X_n$ be a sequence of random variables, such that $X_i \in [0,1]$ with probability 1 and $\mathbb E [X_i|X_1^{i-1}] = b_i$ for $i=1,\dots,n$. Let $Y_1,\dots,Y_n$ be independent Bernoulli random variables, such that $\mathbb E [Y_i] = b_i$. Then for any convex function $f:[0,1]^n \rightarrow \mathbb R:$
\[
\mathbb E \left [f(X_1,\dots,X_n)\right] \leq \mathbb E \left [f(Y_1,\dots,Y_n) \right].
\]
\end{lemma}

Let $\kl(p\|q) = p \ln \frac{p}{q} + (1-p) \ln \frac{1-p}{1-q}$ be an abbreviation for $\KL\left([p, 1-p]\middle\|[q, 1-q]\right)$, where $[p, 1-p]$ and $[q, 1-q]$ are Bernoulli distributions with biases $p$ and $q$, respectively. By Pinsker's inequality \cite{CT91},
\[
|p - q| \leq \sqrt{\kl(p\|q)/2},
\]
which means that a bound on $\kl(p\|q)$ implies a bound on the absolute difference between the biases of the Bernoulli distributions.

We apply Lemma \ref{lem:Martin} in order to derive the following inequality, which is an interesting generalization of an analogous result for i.i.d.\ variables. The result is based on the method of types in information theory \cite{CT91}.
\begin{lemma}
\label{lem:Ekl}
Let $X_1,\dots,X_n$ be a sequence of random variables, such that $X_i \in [0,1]$ with probability 1 and $\mathbb E [X_i|X_1^{i-1}] = b$. Let $S_n := \sum_{i=1}^n X_i$. Then:
\begin{equation}
\label{eq:Ekl}
\mathbb E \left [ e^{n\,\kl \left(\frac{1}{n} S_n \middle\| b \right)} \right ]\leq n+1.
\end{equation}
\end{lemma}
Note
that in Lemma \ref{lem:Ekl} the conditional expectation $\mathbb E[X_i|X_1^{i-1}]$ is identical for all $i$, whereas in Lemma \ref{lem:Martin} there is no such restriction. Combination of Lemma \ref{lem:Ekl} with Markov's inequality leads to the following analog of Hoeffding-Azuma inequality. 
\begin{corollary}
\label{cor:kl}
Let $X_1,\dots,X_n$ be as in Lemma \ref{lem:Ekl}. Then, for any $\delta \in (0,1)$, with probability greater than $1-\delta$:
\begin{equation}
\label{eq:kl}
\kl \left(\frac{1}{n} S_n \middle\| b \right) \leq \frac{1}{n}\ln\frac{n+1}{\delta}.
\end{equation}
\end{corollary}

$S_n$ is a terminal point of a random walk with bias $b$ after $n$ steps. By combining Corollary \ref{cor:kl} with Pinsker's inequality we can obtain a more explicit bound on the deviation of the terminal point from its expected value, $|S_n - bn| \leq \sqrt{\frac{n}{2} \ln \frac{n+1}{\delta}}$, which is similar to the result we can obtain by applying Hoeffding-Azuma's inequality. However, in certain situations the less explicit bound in the form of $\kl$ is significantly tighter than Hoeffding-Azuma's inequality and it can also be tighter than Bernstein's inequality. A detailed comparison is provided in Section \ref{sec:comparison}.

\subsection{PAC-Bayesian Inequalities for Weighted Averages of Martingales}

Next, we present several inequalities that control the concentration of weighted averages of multiple simultaneously evolving and interdependent martingales. The first result shows that the classical PAC-Bayesian theorem for independent random variables \cite{See02} holds in the same form for martingales. The result is based on combination of Donsker-Varadhan's variational formula for relative entropy with Lemma \ref{lem:Ekl}. In order to state the theorem we need a few definitions. 

Let $({\cal H}, {\cal B})$ be a probability space. Let $\bar X_1,\dots,\bar X_n$ be a sequence of random functions, such that $\bar X_i : {\cal H} \rightarrow [0,1]$. Assume that $\mathbb E[\bar X_i| \bar X_1,\dots, \bar X_{i-1}] = \bar b$, where $\bar b : {\cal H} \rightarrow [0,1]$ is a deterministic function (possibly unknown). This means that $\mathbb E[\bar X_i(h)|\bar X_1,\dots,\bar X_{i-1}] = \bar b(h)$ for each $i$ and $h$. Note that for each $h \in {\cal H}$ the sequence $\bar X_1(h), \dots, \bar X_n(h)$ satisfies the condition of Lemma \ref{lem:Ekl}.

Let $\bar S_n := \sum_{i = 1}^n \bar X_i$. In the following theorem we are bounding the mean of $\bar S_n$ with respect to any probability measure $\rho$ over ${\cal H}$.

\begin{theorem}[PAC-Bayes-kl Inequality]
\label{thm:PAC-Bayes-kl}
Fix a reference distribution $\pi$ over ${\cal H}$. Then, for any $\delta \in (0,1)$, with probability greater than $1-\delta$ over $\bar X_1, \dots, \bar X_n$, for all distributions $\rho$ over ${\cal H}$ simultaneously:
\begin{equation}
\kl\left(\left \langle \frac{1}{n} \bar S_n, \rho \right \rangle \middle\| \langle \bar b, \rho \rangle \right) \leq \frac{\KL(\rho\|\pi) + \ln \frac{n+1}{\delta}}{n}.
\label{eq:PAC-Bayes-kl}
\end{equation}
\end{theorem}

By Pinsker's inequality, Theorem \ref{thm:PAC-Bayes-kl} implies that
\begin{align}
\left|\left \langle \frac{1}{n} \bar S_n, \rho \right \rangle - \langle \bar b, \rho\rangle \right| &= \left|\left \langle \left (\frac{1}{n} \bar S_n - \bar b \right ), \rho \right \rangle \right| \notag\\
&\leq \sqrt{\frac{\KL(\rho\|\pi) + \ln \frac{n+1}{\delta}}{2n}},
\label{eq:PAC-Bayes-Pinsker}
\end{align}
however, if $\left \langle \frac{1}{n} \bar S_n, \rho \right \rangle$ is close to zero or one, inequality \eqref{eq:PAC-Bayes-kl} is significantly tighter than \eqref{eq:PAC-Bayes-Pinsker}.

The next result is based on combination of Donsker-Varadhan's variational formula for relative entropy with Hoeffding-Azuma's inequality. This time let $\bar Z_1, \dots, \bar Z_n$ be a sequence of random functions, such that $\bar Z_i : {\cal H} \rightarrow \mathbb R$. Let $\bar Z_1^i$ be an abbreviation for a subsequence of the first $i$ random functions in the sequence. We assume that $\mathbb E[\bar Z_i | \bar Z_1^i] = \bar 0$. In other words, for each $h \in {\cal H}$ the sequence $Z_1(h),\dots,Z_n(h)$ is a martingale difference sequence. 

Let $\bar M_i := \sum_{j=1}^i \bar Z_j$. Then, for each $h \in {\cal H}$ the sequence $\bar M_1(h), \dots, \bar M_n(h)$ is a martingale. In the following theorems we bound the mean of $\bar M_n$ with respect to any probability measure $\rho$ on ${\cal H}$.

\begin{theorem}
\label{thm:PB-HA}
Assume that $\bar Z_i : {\cal H} \rightarrow [\alpha_i, \beta_i]$. Fix a reference distribution $\pi$ over ${\cal H}$ and $\lambda > 0$. Then, for any $\delta \in (0,1)$, with probability greater than $1 - \delta$ over $\bar Z_1^n$, for all distributions $\rho$ over ${\cal H}$ simultaneously:
\begin{equation}
|\langle \bar M_n, \rho \rangle| \leq \frac{\KL(\rho\|\pi) + \ln \frac{2}{\delta}}{\lambda} + \frac{\lambda}{8} \sum_{i=1}^n (\beta_i - \alpha_i)^2.
\label{eq:PB-HA}
\end{equation}
\end{theorem}

We note that we cannot minimize inequality \eqref{eq:PB-HA} simultaneously for all $\rho$ by a single value of $\lambda$. In the following theorem we take a grid of $\lambda$-s in a form of a geometric sequence and for each value of $\KL(\rho\|\pi)$ we pick a value of $\lambda$ from the grid, which is the closest to the one that minimizes \eqref{eq:PB-HA}. The result is almost as good as what we could achieve if we would minimize the bound just for a single value of $\rho$.

\begin{theorem}[PAC-Bayes-Hoeffding-Azuma Inequality]
\label{thm:PB-HA+}
Assume that $\bar Z_1^n$ is as in Theorem \ref{thm:PB-HA}. Fix a reference distribution $\pi$ over ${\cal H}$. Take an arbitrary number $c > 1$. Then, for any $\delta \in (0,1)$, with probability greater than $1 - \delta$ over $\bar Z_1^n$, for all distributions $\rho$ over ${\cal H}$ simultaneously:
\begin{align}
|\langle \bar M_n,& \rho \rangle|\notag\\
&\leq \frac{1+c}{2\sqrt 2}\sqrt{\left (\KL(\rho\|\pi) + \ln \frac{2}{\delta} + \epsilon(\rho)\right )\sum_{i=1}^n (\beta_i - \alpha_i)^2},
\label{eq:PB-HA+}
\end{align}
where
\[
\epsilon(\rho) = \frac{\ln 2}{2 \ln c}\left (1 + \ln \left (\frac{\KL(\rho\|\pi)}{\ln \frac{2}{\delta}} \right ) \right ).
\]
\end{theorem}

Our last result is based on a combination of Donsker-Varadhan's variational formula with a Bernstein-type inequality for martingales. Let $\bar V_i: {\cal H} \rightarrow \mathbb R$ be such that $\bar V_i(h) := \sum_{j=1}^i \mathbb E \left [\bar Z_j(h)^2 \middle|\bar Z_1^{j-1} \right]$. In other words, $\bar V_i(h)$ is the variance of the martingale $\bar M_i(h)$ defined earlier. Let $\|\bar Z_i\|_\infty = \sup_{h \in {\cal H}} \bar Z_i(h)$ be the $L_\infty$ norm of $\bar Z_i$.

\begin{theorem}
\label{thm:PB-B}
Assume that $\|\bar Z_i\|_\infty \leq K$ for all $i$ with probability 1 and pick $\lambda$, such that $\lambda \leq 1/K$. Fix a reference distribution $\pi$ over ${\cal H}$. Then, for any $\delta \in (0,1)$, with probability greater than $1-\delta$ over $\bar Z_1^n$, for all distributions $\rho$ over ${\cal H}$ simultaneously:
\begin{equation}
|\langle \bar M_n, \rho \rangle| \leq \frac{\KL(\rho\|\pi) + \ln \frac{2}{\delta}}{\lambda} + (e-2) \lambda \langle \bar V_n, \rho\rangle.
\label{eq:PB-B}
\end{equation}
\end{theorem}

As in the previous case, the right hand side of \eqref{eq:PB-B} cannot be minimized for all $\rho$ simultaneously by a single value of $\lambda$. Furthermore, $\bar V_n$ is a random function. In the following theorem we take a similar grid of $\lambda$-s, as we did in Theorem \ref{thm:PB-HA+}, and a union bound over the grid. Picking a value of $\lambda$ from the grid closest to the value of $\lambda$ that minimizes the right hand side of \eqref{eq:PB-B} yields almost as good result as we would get if we would minimize \eqref{eq:PB-B} for a single choice of $\rho$. In this approach the variance $\bar V_n$ can be replaced by a sample-dependent upper bound. For example, in importance-weighted sampling such an upper bound is derived from the reciprocal of the sampling distribution at each round \cite{SAL+11}.

\begin{theorem}[PAC-Bayes-Bernstein Inequality]
\label{thm:PB-B+}
Assume that $\|\bar Z_i\|_\infty \leq K$ for all $i$ with probability 1. Fix a reference distribution $\pi$ over ${\cal H}$. Pick an arbitrary number $c > 1$. Then, for any $\delta \in (0,1)$, with probability greater than $1-\delta$ over $\bar Z_1^n$, simultaneously for all distributions $\rho$ over ${\cal H}$  that satisfy
\begin{equation}
\label{eq:technical}
\sqrt{\frac{\KL(\rho\|\pi) + \ln \frac{2\nu}{\delta}}{(e-2) \langle \bar V_n, \rho \rangle}} \leq \frac{1}{K}
\end{equation}
we have
\begin{equation}
|\langle \bar M_n, \rho \rangle| \leq (1+c) \sqrt{(e-2) \langle \bar V_n, \rho \rangle \left (\KL(\rho\|\pi) + \ln \frac{2\nu}{\delta} \right)},
\label{eq:PB-B+}
\end{equation}
where
\begin{equation}
\label{eq:m}
\nu = \left \lceil \frac{\ln \left (\sqrt{\frac{(e-2)n}{\ln \frac{2}{\delta}}} \right )}{\ln c} \right \rceil + 1,
\end{equation}
and for all other $\rho$
\begin{equation}
|\langle \bar M_n, \rho \rangle| \leq 2 K \left ( \KL(\rho\|\pi) + \ln \frac{2\nu}{\delta} \right ).
\label{eq:else}
\end{equation}
\end{theorem}
($\lceil x \rceil$ is the smallest integer value that is larger than $x$.)

\section{Comparison of the Inequalities}
\label{sec:comparison}

In this section we remind the reader of Hoeffding-Azuma's and Bernstein's inequalities for individual martingales and compare them with our new $\kl$-form inequality. Then, we compare inequalities for weighted averages of martingales with inequalities for individual martingales.

\subsection{Background}

We first recall Hoeffding-Azuma's inequality \cite{Hoe63, Azu67}. For a sequence of random variables $Z_1,\dots,Z_n$ we use $Z_1^i := Z_1,\dots,Z_i$ to denote the first $i$ elements of the sequence.
\begin{lemma}[Hoeffding-Azuma's Inequality]
\label{lem:HA}
Let $Z_1,\dots,Z_n$ be a martingale difference sequence, such that $Z_i \in [\alpha_i,\beta_i]$ with probability 1 and $\mathbb E[Z_i|Z_1^{i-1}] = 0$. Let $M_i = \sum_{j=1}^i Z_j$ be the corresponding martingale. Then for any $\lambda \in \mathbb R$:
\[
\mathbb E[e^{\lambda M_n}] \leq e^{(\lambda^2 / 8) \sum_{i=1}^n (\beta_i - \alpha_i)^2}.
\]
\end{lemma}

By combining Hoeffding-Azuma's inequality with Markov's inequality and taking $\lambda = \sqrt{\frac{8\ln \frac{2}{\delta}}{\sum_{i=1}^n (\beta_i-\alpha_i)^2}}$ it is easy to obtain the following corollary.
\begin{corollary}
\label{cor:HA}
For $M_n$ defined in Lemma \ref{lem:HA} and $\delta \in (0,1)$, with probability greater than $1-\delta$:
\[
|M_n| \leq \sqrt{\frac{1}{2}\ln \left (\frac{2}{\delta} \right )\sum_{i=1}^n (\beta_i-\alpha_i)^2}.
\]
\end{corollary}

The next lemma is a Bernstein-type inequality \cite{Ber46, Fre75}. We provide the proof of this inequality in Appendix \ref{app:back}, the proof is a part of the proof of \cite[Theorem 1]{BLL+11}.
\begin{lemma}[Bernstein's Inequality]
\label{lem:Bernstein}
Let $Z_1,\dots,Z_n$ be a martingale difference sequence, such that $|Z_i| \leq K$ with probability 1 and $\mathbb E[Z_i|Z_1^{i-1}] = 0$. Let $M_i := \sum_{j=1}^i Z_j$ and let $V_i := \sum_{j=1}^i \mathbb E[(Z_j)^2|Z_1^{j-1}]$. Then for any $\lambda \in [0,\frac{1}{K}]$:
\[
\mathbb E\left[e^{\lambda M_n - (e-2) \lambda^2 V_n}\right] \leq 1.
\]
\end{lemma}

By combining Lemma \ref{lem:Bernstein} with Markov's inequality we obtain that for any $\lambda \in [0, \frac{1}{K}]$ and $\delta \in (0,1)$, with probability greater than $1-\delta$:
\begin{equation}
\label{eq:lambda}
|M_n| \leq \frac{1}{\lambda}\ln \frac{2}{\delta} + \lambda (e-2) V_n.
\end{equation}

$V_n$ is a random variable and can be replaced by an upper bound. Inequality \eqref{eq:lambda} is minimized by $\lambda^* = \sqrt{\frac{\ln \frac{2}{\delta}}{(e-2) V_n}}$. Note that $\lambda^*$ depends on $V_n$ and is not accessible until we observe the entire sample. We can bypass this problem by constructing the same grid of $\lambda$-s, as the one used in the proof of Theorem \ref{thm:PB-B+}, and taking a union bound over it. Picking a value of $\lambda$ closest to $\lambda^*$ from the grid leads to the following corollary. In this bounding technique the upper bound on $V_n$ can be sample-dependent, since the bound holds simultaneously for all $\lambda$-s in the grid. Despite being a relatively simple consequence of Lemma \ref{lem:Bernstein}, we have not seen this result in the literature. The corollary is tighter than an analogous result by Beygelzimer et. al. \cite[Theorem 1]{BLL+11}.

\begin{corollary}
\label{cor:Bernstein}
For $M_n$ and $V_n$ as defined in Lemma \ref{lem:Bernstein}, $c > 1$ and $\delta \in (0,1)$, with probability greater than $1-\delta$, if 
\begin{equation}
\sqrt{\frac{\ln \frac{2\nu}{\delta}}{(e-2) V_n}} \leq \frac{1}{K}
\label{eq:technical1}
\end{equation}
then 
\[
|M_n| \leq (1+c) \sqrt{(e-2)V_n\ln \frac{2\nu}{\delta}},
\]
where $\nu$ is defined in \eqref{eq:m}, and otherwise
\[
|M_n| \leq 2 K \ln \frac{2\nu}{\delta}.
\]
\end{corollary}

The technical condition \eqref{eq:technical1} follows from the requirement of Lemma \ref{lem:Bernstein} that $\lambda \in [0,\frac{1}{K}]$.

\subsection{Comparison}

We first compare inequalities for individual martingales in Corollaries \ref{cor:kl}, \ref{cor:HA}, and \ref{cor:Bernstein}. 

\subsubsection*{Comparison of Inequalities for Individual Martingales}
The comparison between Corollaries \ref{cor:HA} and \ref{cor:Bernstein} is relatively  straightforward. We note that the assumption $\mathbb E[Z_i|Z_1^{i-1}] = 0$ implies that $\alpha_i \leq 0$ and that $V_n \leq \sum_{i=1}^n \max\{\alpha_i^2,\beta_i^2\} \leq \sum_{i=1}^n (\beta_i - \alpha_i)^2$. Hence, Corollary \ref{cor:Bernstein} (derived from Bernstein's inequality) matches Corollary \ref{cor:HA} (derived from Hoeffding-Azuma's inequality) up to minor constants and logarithmic factors
in the general case, and can be much tighter when the variance is small.

The comparison with the $\kl$ inequality in Corollary \ref{cor:kl} is a bit more involved. As we mentioned after Corollary \ref{cor:kl}, its combination with Pinsker's inequality implies that $|S_n - bn| \leq \sqrt{\frac{n}{2} \ln \frac{n+1}{\delta}}$, where $S_n - bn$ is a martingale corresponding to the martingale difference sequence $Z_i = X_i - b$. Thus, Corollary \ref{cor:kl} is at least as tight as Hoeffding-Azuma's inequality in Corollary \ref{cor:HA}, up to a factor of $\sqrt{\ln \frac{n+1}{2}}$. This is also true if $X_i \in [\alpha_i,\beta_i]$ (rather than $[0,1]$), as long as we can simultaneously project all $X_i$-s to the $[0,1]$ interval without losing too much.

Tighter upper bounds on the $\kl$ divergence show that in certain situations Corollary \ref{cor:kl} is actually much tighter than Hoeffding-Azuma's inequality. One possible application of Corollary \ref{cor:kl} is estimation of the value of the drift $b$ of a random walk from empirical observation $S_n$. If $S_n$ is close to zero, it is possible to use a tighter bound on $\kl$, which states that for $p > q$ we have $p \leq q + \sqrt{2 q\, \kl(q||p)} + 2 \kl(q||p)$ \cite{McA03}. From this inequality, we obtain that with probability greater than $1-\delta$:
\[
b \leq \frac{1}{n} S_n + \sqrt{\frac{\frac{2}{n} S_n \ln \frac{n+1}{\delta}}{n}} + \frac{2 \ln \frac{n+1}{\delta}}{n}.
\]
The above inequality is tighter than Hoeffding-Azuma inequality whenever $\frac{1}{n} S_n < 1/8$. Since $\kl$ is convex in each of its parameters, it is actually easy to invert it numerically, and thus avoid the need to resort to approximations in practice. In a similar manner, tighter bounds can be obtained when $S_n$ is close to $n$.

The comparison of $\kl$ inequality in Corollary \ref{cor:kl} with Bernstein's inequality in Corollary \ref{cor:Bernstein} is not as equivocal as the comparison with Hoeffding-Azuma's inequality. If there is a bound on $V_n$ that is significantly tighter than $n$, Bernstein's inequality can be significantly tighter than the $\kl$ inequality, but otherwise it can also be the opposite case. In the example of estimating a drift of a random walk without prior knowledge on its variance, if the empirical drift is close to zero or to $n$ the $\kl$ inequality is tighter. In this case the $\kl$ inequality is comparable with empirical Bernstein's bounds \cite{MSA08,AMS09,MP09}.

\subsubsection*{Comparison of Inequalities for Individual Martingales with PAC-Bayesian Inequalities for Weighted Averages of Martingales}

The ``price'' that is paid for considering weighted averages of multiple martingales is the KL-divergence $\KL(\rho\|\pi)$ between the desired mixture weights $\rho$ and the reference mixture weights $\pi$. (In the case of PAC-Bayes-Hoeffding-Azuma inequality, Theorem \ref{thm:PB-HA+}, there is also an additional minor term originating from the union bound over the grid of $\lambda$-s.) Note that for $\rho = \pi$ the KL term vanishes.

\section{Discussion}

We presented a comparison inequality that bounds expectation of a convex function of martingale difference type variables by expectation of the same function of independent Bernoulli variables. This inequality enables to reduce a problem of studying continuous dependent random variables on a bounded interval to a much simpler problem of studying independent Bernoulli random variables.

As an example of an application of our lemma we derived an analog of Hoeffding-Azuma's inequality for martingales. Our result is always comparable to Hoeffding-Azuma's inequality up to a logarithmic factor and in cases, where the empirical drift of a corresponding random walk is close to the region boundaries it is tighter than Hoeffding-Azuma's inequality by an order of magnitude. It can also be tighter than Bernstein's inequality for martingales, unless there is a tight bound on the martingale variance.

Finally, but most importantly, we presented a set of inequalities on concentration of weighted averages of multiple simultaneously evolving and interdependent martingales. These inequalities are especially useful for controlling uncountably many martingales, where standard union bounds cannot be applied. Martingales are one of the most basic and important tools for studying time-evolving processes and we believe that our results will be useful for multiple domains. One such application in analysis of importance weighted sampling in reinforcement learning was already presented in \cite{SAL+11}.


%

\appendices
\section{Proofs of the Results for Individual Martingales}

\begin{proof}[Proof of Lemma \ref{lem:Martin}]
The proof follows the lines of the proof of Maurer \cite[Lemma 3]{Mau04}. Any point $\bar x = (x_1,\dots,x_n) \in [0,1]^n$ can be written as a convex combination of the extreme points $\bar \eta = (\eta_1,\dots,\eta_n) \in \{0,1\}^n$ in the following way:
\[
\bar x = \sum_{\bar \eta \in \{0,1\}^n} \left (\prod_{i=1}^n [(1 - x_i)(1 - \eta_i) + x_i \eta_i ]\right ) \bar \eta.
\]
Convexity of $f$ therefore implies
\begin{equation}
f(\bar x) \leq \sum_{\bar \eta \in \{0,1\}^n} \left (\prod_{i=1}^n [(1 - x_i)(1 - \eta_i) + x_i \eta_i ]\right ) f(\bar \eta)
\label{eq:convexity}
\end{equation}
with equality if $\bar x \in \{0,1\}^n$. Let $X_1^i := X_1,\dots,X_i$ be the first $i$ elements of the sequence $X_1,\dots,X_n$. Let $W_i(\eta_i) = (1 - X_i) (1 - \eta_i) + X_i \eta_i$ and let $w_i(\eta_i) = (1 - b_i) (1 - \eta_i) + b_i \eta_i$. Note that by the assumption of the lemma:
\begin{align*}
\mathbb E [W_i(\eta_i)|X_1^{i-1}] &= \mathbb E [(1 - X_i) (1 - \eta_i) + X_i \eta_i |X_1^{i-1}]\\
&= (1 - b_i) (1 - \eta_i) + b_i \eta_i = w_i(\eta_i).
\end{align*}
By taking expectation of both sides of \eqref{eq:convexity} we obtain:
\begin{align}
&\mathbb E_{X_1^n} [f(X_1^n)] \leq \mathbb E_{X_1^n} \left [ \sum_{\bar \eta \in \{0,1\}^n} \left (\prod_{i=1}^n W_i(\eta_i) \right ) f(\bar \eta) \right ]\notag\\
&= \sum_{\bar \eta \in \{0,1\}^n} \mathbb E_{X_1^n} \left [ \prod_{i=1}^n W_i(\eta_i) \right ] f(\bar \eta) \notag\\
&= \sum_{\bar \eta \in \{0,1\}^n} \mathbb E_{X_1^{n-1}} \left [ \mathbb E_{X_n} \left [ \left . \prod_{i=1}^n W_i(\eta_i) \right | X_1^{n-1}\right ]\right ]f(\bar \eta)\notag\\
&= \sum_{\bar \eta \in \{0,1\}^n} \mathbb E_{X_1^{n-1}} \left [\prod_{i=1}^{n-1} W_i(\eta_i) \mathbb E_{X_n} \left [W_n(\eta_n)| X_1^{n-1}\right ] \right ]f(\bar \eta)\notag\\
&= \sum_{\bar \eta \in \{0,1\}^n} \mathbb E_{X_1^{n-1}} \left [ \prod_{i=1}^{n-1} W_i(\eta_i) \right ] w_n(\eta_n) f(\bar \eta)\notag\\
&= \dots \label{eq:induction}\\
&= \sum_{\bar \eta \in \{0,1\}^n} \left (\prod_{i=1}^n w_i(\eta_i) \right ) f(\bar \eta)\notag\\
&= \sum_{\bar \eta \in \{0,1\}^n} \left (\prod_{i=1}^n [(1 - b_i)(1 - \eta_i) + b_i \eta_i] \right ) f(\bar \eta)\notag\\
&= \mathbb E_{Y_1^n} [f(Y_1^n)].\notag
\end{align}
In \eqref{eq:induction} we apply induction in order to replace $X_i$ by $b_i$, one-by-one from the last to the first, same way we did it for $X_n$.
\end{proof}

Lemma \ref{lem:Ekl} follows from the following concentration result for independent Bernoulli variables that is based on the method of types in information theory \cite{CT91}. Its proof can be found in \cite{See03,ST10}.
\begin{lemma}
\label{lem:Laplace}
Let $Y_1,\dots,Y_n$ be i.i.d. Bernoulli random variables, such that $\mathbb E[Y_i] = b$. Then:
\begin{equation}
\mathbb E \left[e^{n\,\kl\left(\frac{1}{n} \sum_{i=1}^n Y_i\middle\|b\right)}\right] \leq n+1. 
\label{eq:Laplace}
\end{equation}
\end{lemma}

For $n\geq8$ it is possible to prove even stronger result $\sqrt n \leq \mathbb E[e^{n \, \kl(\frac{1}{n} \sum_{i=1}^n Y_i\|b)}] \leq 2 \sqrt n$  using Stirling's approximation of the factorial \cite{Mau04}. For the sake of simplicity we restrict ourselves to the slightly weaker bound \eqref{eq:Laplace}, although all results that are based on Lemma \ref{lem:Ekl} can be slightly improved by using the tighter bound.

\begin{proof}[Proof of Lemma \ref{lem:Ekl}]
Since KL-divergence is a convex function \cite{CT91} and the exponent function is convex and non-decreasing, $e^{n \, \kl(p\|q)}$ is also a convex function. Therefore, Lemma \ref{lem:Ekl} follows from Lemma \ref{lem:Laplace} by Lemma \ref{lem:Martin}.
\end{proof}

Corollary \ref{cor:kl} follows from Lemma \ref{lem:Ekl} by Markov's inequality.
\begin{lemma}[Markov's inequality]
\label{lem:Markov}
For $\delta \in (0,1)$ and a random variable $X \geq 0$, with probability greater than $1-\delta$$:$
\begin{equation}
X \leq \frac{1}{\delta} \mathbb E[X].
\end{equation}
\end{lemma}

\begin{proof}[Proof of Corollary \ref{cor:kl}]
By Markov's inequality and Lemma \ref{lem:Ekl}, with probability greater than $1-\delta$:
\[
e^{n\, \kl\left(\frac{1}{n} S_n\middle\|b\right)} \leq \frac{1}{\delta} \mathbb E \left[e^{n\, \kl\left(\frac{1}{n} S_n\middle\|b\right)}\right] \leq \frac{n+1}{\delta}.
\]
Taking logarithm of both sides of the inequality and normalizing by $n$ completes the proof.
\end{proof}

\section{Proofs of PAC-Bayesian Theorems for Martingales}

In this appendix we provide the proofs of Theorems \ref{thm:PAC-Bayes-kl}, \ref{thm:PB-B}, and \ref{thm:PB-B+}. The proof of Theorem \ref{thm:PB-HA} is very similar to the proof of Theorem \ref{thm:PB-B} and, therefore, omitted. The proof of Theorem \ref{thm:PB-HA+} is very similar to the proof of Theorem \ref{thm:PB-B+}, so we only provide the way of how to choose the grid of $\lambda$-s in this theorem.

The proofs of all PAC-Bayesian theorems are based on the following lemma, which is obtained by changing sides in Donsker-Varadhan's variational definition of relative entropy. The lemma takes roots back in information theory and statistical physics \cite{DV75, DE97, Gra11}. The lemma provides a deterministic relation between averages of $\phi$ with respect to all possible distributions $\rho$ and the cumulant generating function $\ln \langle e^\phi, \pi \rangle$ with respect to a single reference distribution $\pi$. A single application of Markov's inequality combined with the bounds on moment generating functions in Lemmas \ref{lem:Ekl}, \ref{lem:HA}, and \ref{lem:Bernstein} is then used in order to bound the last term in \eqref{eq:PAC-Bayes} in the proofs of Theorems \ref{thm:PAC-Bayes-kl}, \ref{thm:PB-HA}, and \ref{thm:PB-B}, respectively.

\begin{lemma}[Change of Measure Inequality]
\label{lem:PAC-Bayes}
For any probability space $({\cal H}, {\cal B})$, a measurable function $\phi:{\cal H} \rightarrow \mathbb R$, and any distributions $\pi$ and $\rho$ over ${\cal H}$, we have$:$
\begin{equation}
\langle \phi, \rho \rangle \leq \KL(\rho\|\pi) + \ln \langle e^\phi, \pi \rangle.
\label{eq:PAC-Bayes}
\end{equation}
\end{lemma}

Since the KL-divergence is infinite when the support of $\rho$ exceeds the support of $\pi$, inequality \eqref{eq:PAC-Bayes} is interesting when $\pi \gg \rho$. For a similar reason, it is interesting only when $\langle e^\phi, \pi \rangle$ is finite. We note that the inequality is tight in the same sense as Jensen's inequality is tight: for $\phi(h) = \ln \frac{\rho(h)}{\pi(h)}$ it becomes an equality.

\begin{proof}[Proof of Theorem \ref{thm:PAC-Bayes-kl}]
Take $\phi(h) := n \, \kl\left(\frac{1}{n} \bar S_n(h)\middle\| \bar b(h)\right)$. More compactly, denote $\phi = \kl \left(\frac{1}{n} \bar S_n\middle\| \bar b \right ): {\cal H} \rightarrow \mathbb R$. Then with probability greater than $1-\delta$ for all $\rho$:
\begin{align}
\nonumber
n \,\kl&\left(\left \langle \frac{1}{n} \bar S_n, \rho \right \rangle \,\middle\|\, \langle \bar b, \rho \rangle \right)\\
&\leq n \left \langle \kl\left(\frac{1}{n} \bar S_n\,\middle\|\, \bar b\right), \rho \right \rangle \label{eq:2}\\
&\leq \KL(\rho\|\pi) + \ln \left \langle e^{n \, \kl(\frac{1}{n} \bar S_n\| \bar b)}, \pi \right \rangle\label{eq:3}\\
&\leq \KL(\rho\|\pi) + \ln \left (\frac{1}{\delta} \mathbb E_{\bar X_1^n} \left [ \left \langle e^{n \, \kl(\frac{1}{n} \bar S_n\| \bar b)}, \pi \right \rangle \right] \right )\label{eq:4}\\
&= \KL(\rho\|\pi) + \ln \left (\frac{1}{\delta} \left \langle \mathbb E_{\bar X_1^n} \left[e^{n \, \kl(\frac{1}{n} \bar S_n\|\bar b)}\right], \pi \right \rangle \right )\label{eq:5}\\
&\leq \KL(\rho\|\pi) + \ln \frac{n+1}{\delta},\label{eq:6}
\end{align}
where \eqref{eq:2} is by convexity of the $\kl$ divergence \cite{CT91}; \eqref{eq:3} is by change of measure inequality (Lemma \ref{lem:PAC-Bayes}); \eqref{eq:4} holds with probability greater than $1-\delta$ by Markov's inequality; in \eqref{eq:5} we can take the expectation inside the dot product due to linearity of both operations and since $\pi$ is deterministic; and \eqref{eq:6} is by Lemma \ref{lem:Ekl}.\footnote{By Lemma \ref{lem:Ekl}, for each $h \in {\cal H}$ we have $\mathbb E_{\bar X_1^n} \left[e^{n \, \kl(\frac{1}{n} \bar S_n(h)\|\bar b(h))}\right] \leq n+1$ and, therefore, $\left \langle \mathbb E_{\bar X_1^n} \left[e^{n \, \kl(\frac{1}{n} \bar S_n\|\bar b)}\right], \pi \right \rangle \leq n+1$.} Normalization by $n$ completes the proof of the theorem.
\end{proof}

\begin{proof}[Proof of Theorem \ref{thm:PB-B}]
For the proof of Theorem \ref{thm:PB-B} we take $\phi(h) := \lambda \bar M_n(h) - (e-2) \lambda^2 \bar V_n(h)$. Or, more compactly, $\phi = \lambda \bar M_n - (e-2) \lambda^2 \bar V_n$. Then with probability greater than $1 - \frac{\delta}{2}$ for all $\rho$:
\begin{align}
\lambda \langle \bar M_n,& \rho \rangle - (e-2) \lambda^2 \langle \bar V_n, \rho \rangle =\langle \lambda \bar M_n - (e-2) \lambda^2 \bar V_n, \rho \rangle \notag\\
&\leq \KL(\rho\|\pi) + \ln \left \langle e^{\lambda \bar M_n - (e-2) \lambda^2 \bar V_n}, \pi \right \rangle\notag\\
&\leq \KL(\rho\|\pi) + \ln \left (\frac{2}{\delta} \mathbb E_{\bar Z_1^n} \left [\left \langle e^{\lambda \bar M_n - (e-2) \lambda^2 \bar V_n}, \pi \right \rangle \right] \right )\label{eq:23}\\
&= \KL(\rho\|\pi) + \ln \left (\frac{2}{\delta} \left \langle \mathbb E_{\bar Z_1^n} \left [ e^{\lambda \bar M_n - (e-2) \lambda^2 \bar V_n} \right], \pi \right \rangle \right )\notag\\
&\leq \KL(\rho\|\pi) + \ln \frac{2}{\delta},\label{eq:26}
\end{align}
where \eqref{eq:26} is by Lemma \ref{lem:Bernstein} and other steps are justified in the same way as in the previous proof.

By applying the same argument to $-\bar M_n$, taking a union bound over the two results, taking $(e-2) \lambda^2 \langle \bar V_n, \rho \rangle$ to the other side of the inequality, and normalizing by $\lambda$, we obtain the statement of the theorem.
\end{proof}

\begin{proof}[Proof of Theorem \ref{thm:PB-B+}]
The value of $\lambda$ that minimizes \eqref{eq:PB-B} depends on $\rho$, whereas we would like to have a result that holds for all possible distributions $\rho$ simultaneously. This requires considering multiple values of $\lambda$ simultaneously and we have to take a union bound over $\lambda$-s in step \eqref{eq:23} of the proof of Theorem \ref{thm:PB-B}. We cannot take all possible values of $\lambda$, since there are uncountably many possibilities. Instead we determine the relevant range of $\lambda$ and take a union bound over a grid of $\lambda$-s that forms a geometric sequence over this range. Since the range is finite, the grid is also finite.

The upper bound on the relevant range of $\lambda$ is determined by the constraint that $\lambda \leq \frac{1}{K}$. For the lower bound we note that since $\KL(\rho\|\pi) \geq 0$, the value of $\lambda$ that minimizes \eqref{eq:PB-B} is lower bounded by $\sqrt{\frac{\ln \frac{2}{\delta}}{(e-2) \langle \bar V_n, \rho\rangle}}$. We also note that $\langle \bar V_n, \rho \rangle \leq K^2 n$, since $|Z_i(h)| \leq K$ for all $h$ and $i$. Hence, $\lambda \geq \frac{1}{K} \sqrt{\frac{\ln \frac{2}{\delta}}{(e-2)n}}$ and the range of $\lambda$ we are interested in is
\[
\lambda \in \left[\frac{1}{K} \sqrt{\frac{\ln \frac{2}{\delta}}{(e-2)n}}, \frac{1}{K}\right].
\]
We cover the above range with a grid of $\lambda_i$-s, such that $\lambda_i := c^i \frac{1}{K} \sqrt{\frac{\ln \frac{2}{\delta}}{(e-2)n}}$ for $i = 0,\dots,m-1$. It is easy to see that in order to cover the interval of relevant $\lambda$ we need
\[
m = \left \lceil \frac{1}{\ln c}\ln \left ( \sqrt{\frac{(e-2)n}{\ln \frac{2}{\delta}}} \right ) \right \rceil.
\]
($\lambda_{m-1}$ is the last value that is strictly less than $1/K$ and we take $\lambda_m := 1/K$ for the case when the technical condition \eqref{eq:technical} is not satisfied). This defines the value of $\nu$ in \eqref{eq:m}.

Finally, we note that \eqref{eq:PB-B} has the form $g(\lambda) = \frac{U}{\lambda} + \lambda V$. For the relevant range of $\lambda$, there is $\lambda_{i^*}$ that satisfies $\sqrt{U/V} \leq \lambda_{i^*} < c \sqrt{U/V}$. For this value of $\lambda$ we have $g(\lambda_{i^*}) \leq (1+c) \sqrt{UV}$. 

Therefore, whenever \eqref{eq:technical} is satisfied we pick the highest value of $\lambda_i$ that does not exceed the left hand side of \eqref{eq:technical}, substitute it into \eqref{eq:PB-B}, and obtain \eqref{eq:PB-B+}, where the $\ln \nu$ factor comes from the union bound over $\lambda_i$-s. If \eqref{eq:technical} is not satisfied, we know that $\langle \bar V_n, \rho \rangle < K^2 \left (KL(\rho\|\pi) + \ln \frac{2\nu}{\delta}\right) / (e-2)$ and by taking $\lambda = 1 / K$ and substituting into \eqref{eq:PB-B} we obtain \eqref{eq:else}.
\end{proof}

\begin{proof}[Proof of Theorem \ref{thm:PB-HA+}]
Theorem \ref{thm:PB-HA+} follows from Theorem \ref{thm:PB-HA} in the same way as Theorem \ref{thm:PB-B+} follows from Theorem \ref{thm:PB-B}. The only difference is that the relevant range of $\lambda$ is unlimited from above. If $\KL(\rho\|\pi) = 0$ the bound is minimized by
\[
    \lambda = \sqrt{\frac{8 \ln \frac{2}{\delta}}{\sum_{i=1}^n (\beta_i - \alpha_i)^2}},
\]
hence, we are interested in $\lambda$ that is larger or equal to this value. We take a grid of $\lambda_i$-s of the form
\[
\lambda_i := c^i\sqrt{\frac{8 \ln \frac{2}{\delta}}{\sum_{i=1}^n (\beta_i - \alpha_i)^2}}
\]
for $i \geq 0$. Then for a given value of $\KL(\rho\|\pi)$ we have to pick $\lambda_i$, such that
\[
i = \left\lfloor \frac{\ln \left (\frac{\KL(\rho\|\pi)}{\ln \frac{2}{\delta}} + 1 \right )}{2 \ln c} \right\rfloor,
\]
where $\lfloor x \rfloor$ is the largest integer value that is smaller than $x$. Taking a weighted union bound over $\lambda_i$-s with weights $2^{-(i+1)}$ completes the proof. (In the weighted union bound we take $\delta_i = \delta 2^{-(i+1)}$. Then by substitution of $\delta$ with $\delta_i$, \eqref{eq:PB-HA} holds with probability greater than $1-\delta_i$ for each $\lambda_i$ individually, and with probability greater than $1 - \sum_{i=0}^\infty \delta_i = 1 - \delta$ for all $\lambda_i$ simultaneously.)
\end{proof}

\section{Background}
\label{app:back}

In this section we provide a proof of Lemma \ref{lem:Bernstein}. The proof reproduces an intermediate step in the proof of \cite[Theorem 1]{BLL+11}.
\begin{proof}[Proof of Lemma \ref{lem:Bernstein}]
First, we have:
\begin{align}
\mathbb E_{Z_i} \left [e^{\lambda Z_i} \middle | Z_1^{i-1} \right] &\leq \mathbb E_{Z_i} \left [1 + \lambda Z_i + (e-2) \lambda^2 (Z_i)^2 \middle | Z_1^{i-1} \right]\label{eq:31}\\
&= 1 + (e-2) \lambda^2 \mathbb E_{Z_i} \left [ (Z_i)^2 \middle | Z_1^{i-1}\right ]\label{eq:32}\\
&\leq e^{(e-2) \lambda^2 \mathbb E_{Z_i} \left [ (Z_i)^2 \middle | Z_1^{i-1}\right ]},\label{eq:33}
\end{align}
where \eqref{eq:31} uses the fact that $e^x \leq 1 + x + (e-2) x^2$ for $x \leq 1$ (this restricts the choice of $\lambda$ to $\lambda \leq \frac{1}{K}$, which leads to technical conditions \eqref{eq:technical} and \eqref{eq:technical1} in Theorem \ref{thm:PB-B+} and Corollary \ref{cor:Bernstein}, respectively); \eqref{eq:32} uses the martingale property $\mathbb E_{Z_i}[Z_i | Z_1^{i-1}] = 0$; and \eqref{eq:33} uses the fact that $1 + x \leq e^x$ for all $x$.

We apply inequality \eqref{eq:33} in the following way:
\begin{align}
&\mathbb E_{Z_1^n}\left[e^{\lambda M_n - (e-2) \lambda^2 V_n}\right]\notag\\
&= \mathbb E_{Z_1^n}\left[e^{\lambda M_{n-1} - (e-2) \lambda^2 V_{n-1} + \lambda Z_n - (e-2) \lambda^2 \mathbb E \left[(Z_n)^2\middle|Z_1^{n-1}\right]} \right] \notag\\
&= \mathbb E_{Z_1^{n-1}}\left[
\begin{array}{l}
e^{\lambda M_{n-1} - (e-2) \lambda^2 V_{n-1}}\notag\\
\, \times \, \mathbb E_{Z_n} \left [e^{\lambda Z_n}\middle | Z_1^{n-1} \right] \times e^{-(e-2) \lambda^2 \mathbb E \left[(Z_n)^2\middle|Z_1^{n-1}\right]}
\end{array}
\right]\label{eq:34}\\
&\leq \mathbb E_{Z_1^{n-1}}\left[e^{\lambda M_{n-1} - (e-2) \lambda^2 V_{n-1}} \right ]\\
&\leq \dots \label{eq:35}\\
&\leq 1.\notag
\end{align}
Inequality \eqref{eq:34} applies inequality \eqref{eq:33} and inequality \eqref{eq:35} recursively proceeds with $Z_{n-1},\dots,Z_1$ (in reverse order).
\end{proof}

Note that conditioning on additional variables in the proof of the lemma does not change the result. This fact is exploited in the proof of Theorem \ref{thm:PB-B}, when we allow interdependence between multiple martingales.

\section*{Acknowledgments}

The authors would like to thank Andreas Maurer for his comments on Lemma \ref{lem:Martin}. We are also very grateful to the anonymous reviewers for their valuable comments that helped to improve the presentation of our work. This work was supported in part by the IST Programme of the European Community, under the PASCAL2 Network of Excellence, IST-2007-216886, and by the European Community's Seventh Framework Programme (FP7/2007-2013), under grant agreement $N^o$270327. This publication only reflects the authors' views.

\ifCLASSOPTIONcaptionsoff
  \newpage
\fi



\bibliographystyle{IEEEtran}
\bibliography{IEEEabrv,bibliography}

%

\begin{biography}[{\includegraphics[width=1in,height=1.25in,clip,keepaspectratio]{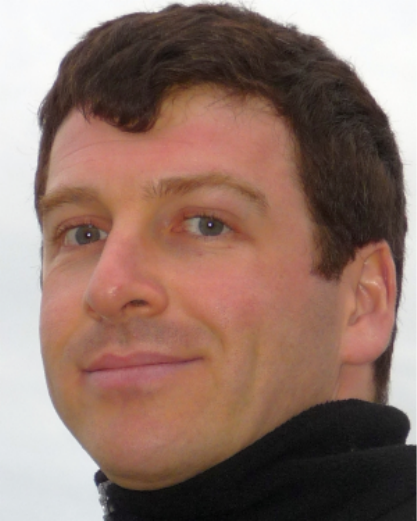}}]{Yevgeny Seldin}
received his Ph.D.\ in computer science from the Hebrew University of Jerusalem in 2010. Since 2009 he is a Research Scientist at the Max Planck Institute for Intelligent Systems in T{\" u}bingen and since 2011 he is also an Honorary Research Associate at the Department of Computer Science in University College London. His research interests include statistical learning theory, PAC-Bayesian analysis, and reinforcement learning. He has contributions in PAC-Bayesian analysis, reinforcement learning, clustering-based models in supervised and unsupervised learning, collaborative filtering, image processing, and bioinformatics.
\end{biography}

\begin{biography}[{\includegraphics[width=1in,height=1.25in,clip,keepaspectratio]{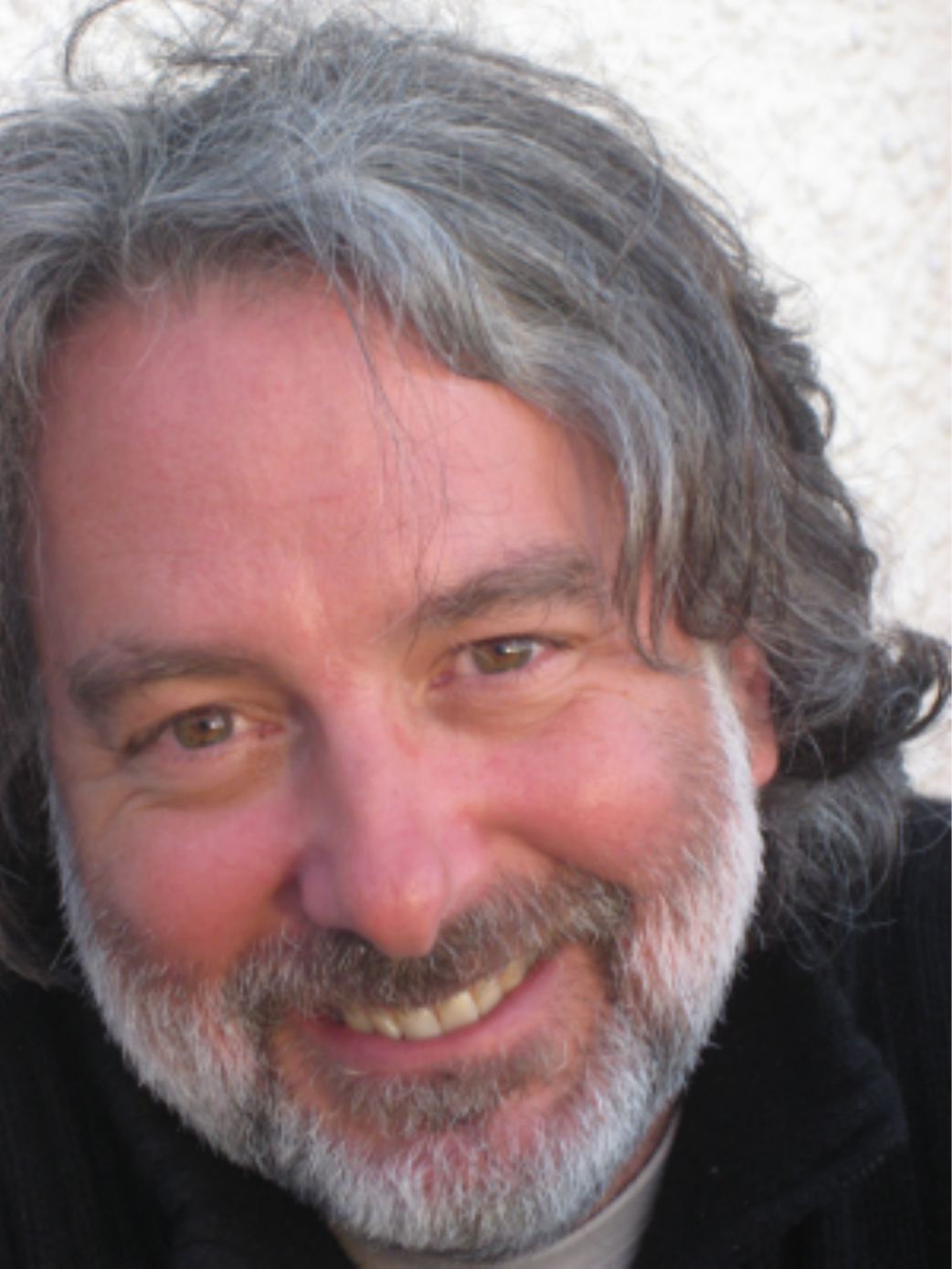}}]{Fran\c{c}ois Laviolette}
received his Ph.D.\ in mathematics from Universit\'e de Montr\'eal in 1997. His thesis solved a long-standing conjecture (60 years old) on graph theory and was among the seven finalists of the 1998 Council of Graduate Schools / University Microfilms International Distinguished Dissertation Award of Washington, in the category Mathematics-Physic-Engineering. He then moved to Universit\'e Laval, where he works on Probabilistic Verification of Systems, Bio-Informatics, and Machine Learning, with a particular interest in PAC-Bayesian analysis, for which he has already more than a dozen of scientific publications. 
\end{biography}

\begin{biography}[{\includegraphics[width=1in,height=1.25in,clip,keepaspectratio]{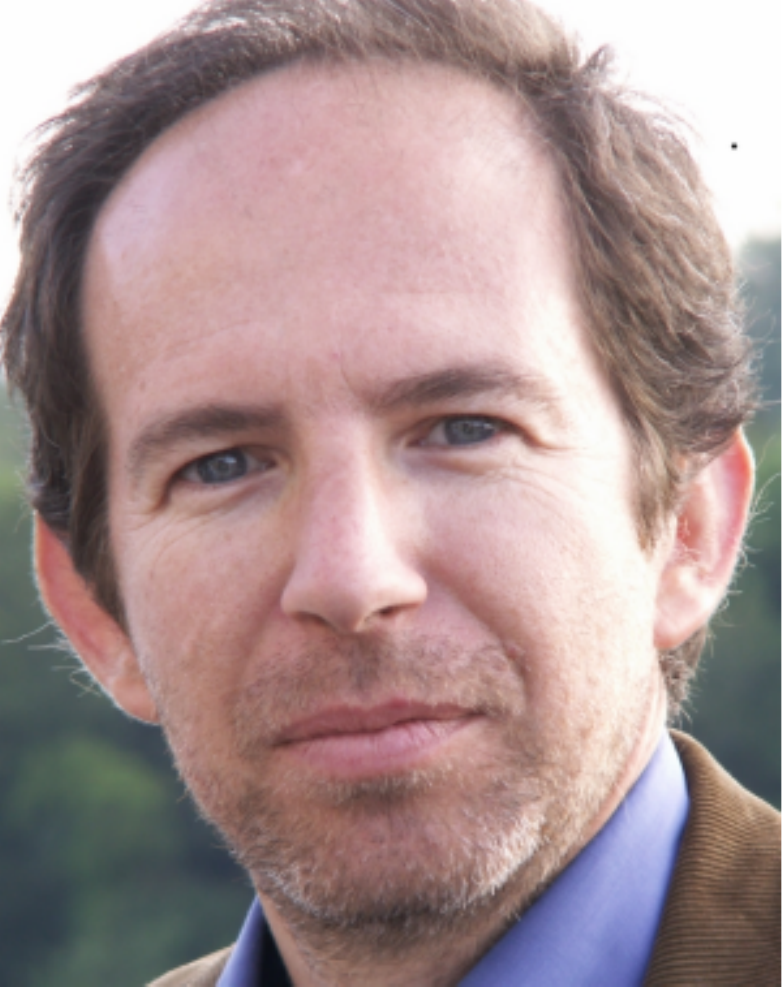}}]{Nicol{\`o} Cesa-Bianchi}
is a faculty member of the Computer Science Department at the Universit\`a degli Studi di Milano, Italy. His main research interests include statistical learning theory, game-theoretic learning, and pattern analysis. He is co-author with G\`abor Lugosi of the monography ``Prediction, Learning, and Games'' (Cambridge University Press, 2006).
\end{biography}

\begin{biography}[{\includegraphics[width=1in,height=1.25in,clip,keepaspectratio]{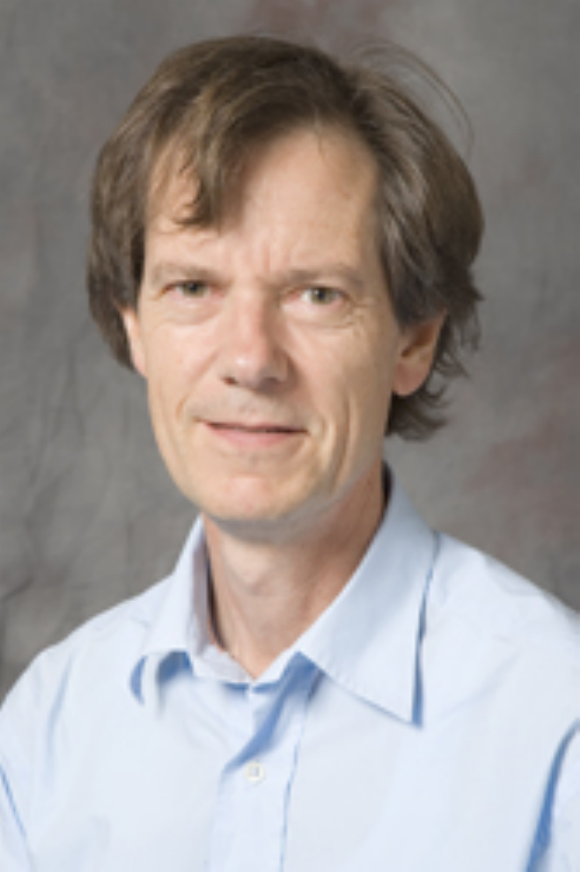}}]{John Shawe-Taylor}
obtained a Ph.D.\ in Mathematics at Royal Holloway, University of London in 1986. He subsequently completed an M.Sc.\ in the Foundations of Advanced Information Technology at Imperial College. He was promoted to Professor of Computing Science in 1996. He has published over 200 research papers. In 2006 he was appointed Director of the Center for Computational Statistics and Machine Learning at University College London. He has pioneered the development of the well-founded approaches to Machine Learning inspired by statistical learning theory (including Support Vector Machine, Boosting and Kernel Principal Components Analysis) and has shown the viability of applying these techniques to document analysis and computer vision. 
He is co-author of an Introduction to Support Vector Machines, the first comprehensive account of this new generation of machine learning algorithms. A second book on Kernel Methods for Pattern Analysis was published in 2004.
\end{biography}

\begin{biography}[{\includegraphics[width=1in,height=1.25in,clip,keepaspectratio]{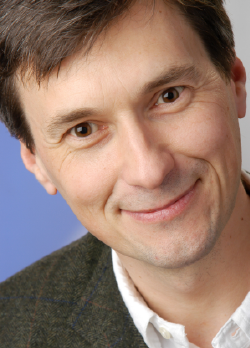}}]{Peter Auer}
received his Ph.D.\ in mathematics from the Vienna University of Technology in 1992, working on probability theory with Pal Revesz and on Symbolic Computation with Alexander Leitsch. He then moved to Graz University of Technology, working on Machine Learning with Wolfgang Maass, and was appointed associate professor in 1997. He has also been a research scholar at the University of California, Santa Cruz. In 2003 he accepted the position of a full professor for Information Technology at the Montanuniversit{\"a}t Leoben. He has authored scientific publications in the areas of probability theory, symbolic computation, and machine learning, he is a member of the editorial board of Machine Learning, and he has been principal investigator in several research projects funded by the European Union. His current research interests include Machine Learning focused on autonomous learning and exploration algorithms.
\end{biography}





\end{document}